\documentclass{article}

\usepackage{amsmath,amsthm}

\usepackage{microtype}
\usepackage{graphicx}
\usepackage{subfigure}
\usepackage{booktabs} % for professional tables
\usepackage{footnote}
\usepackage{graphicx}
\usepackage{hyperref}
\usepackage{booktabs} % for professional tables

\usepackage{url}
\usepackage{amsfonts}
\usepackage{amssymb}
\usepackage{cite}

\usepackage{tikz}
\usetikzlibrary{positioning,arrows.meta}
\usetikzlibrary{shapes.geometric, arrows}
\usetikzlibrary{patterns, arrows.meta}
\tikzstyle{output} = [rectangle, rounded corners, minimum width=1cm, minimum height=1cm,text centered, draw=black, fill=gray!30]
\tikzstyle{hidden} = [circle, minimum width=0.3cm, minimum height=1cm, text centered, draw=black, fill=gray!100]
\tikzstyle{input} = [diamond, minimum width=0.3cm, minimum height=.25cm, text centered, draw=black, fill=gray!30]
\tikzstyle{arrow} = [thick,->,>=stealth]
\tikzstyle{rcircle} =[circle, minimum width=0.05cm, fill=red!60]

\usepackage{footnote}
\newcommand{\vect}[1]{\boldsymbol{#1}}
\newcommand{\defeq}{\overset{\text{def}}{=}}

\newcommand{\cU}{\mathcal{U}}
\newcommand{\cV}{\mathcal{V}}
\newcommand{\cH}{{\mathcal{H}}_f}

\newcommand{\cNN}{f}

\newtheorem{definition}{Definition}
\newtheorem{theorem}{Theorem}
\newtheorem{lemma}{Lemma}
\newtheorem{proposition}{Proposition}
\newtheorem{corollary}{Corollary}

\DeclareMathOperator*{\argmax}{arg\,max}
\usepackage[accepted]{icml2018}

\icmltitlerunning{Bounds on the Approximation Power of Feedforward Neural Networks}
\begin{document}

\twocolumn[
\icmltitle{Bounds on the Approximation Power of Feedforward Neural Networks}



\begin{icmlauthorlist}
\icmlauthor{Mohammad Mehrabi}{to}
\icmlauthor{Aslan Tchamkerten}{goo}
\icmlauthor{Mansoor I. Yousefi}{goo}
\end{icmlauthorlist}

\icmlaffiliation{to}{Department of Electrical Engineering, Sharif University, Iran}
\icmlaffiliation{goo}{Department of Communications and Electronics, Telecom ParisTech, France}

\icmlcorrespondingauthor{Mohammad Mehrabi}{mohamadmehrabi4@gmail.com}

\icmlkeywords{Neural networks, approximation power}

\vskip 0.3in
]



\printAffiliationsAndNotice{} % otherwise use the standard text.

\begin{abstract}
%Expressiveness of general feedforward neural networks with piecewise linear activation functions is investigated. Lower bounds on the approximating power are established in terms depth and number of hidden units. These bounds improve previously known bounds for certain classes of functions, such as strongly convex functions. The second part of the paper addresses the problem of quantifying the change of a neural network when its activation function is modified.

The approximation power of general feedforward neural networks with piecewise linear activation functions is investigated. First, lower bounds on the size of a network are established in terms of the approximation error and network depth and width. These bounds improve upon state-of-the-art bounds for certain classes of functions, such as strongly convex functions. Second, an upper bound is established on the 
difference of two neural networks with identical weights but different activation functions. 
%changes in the neural network output as the activation function is modified.

\end{abstract}

\section{Introduction}
\label{introduction}

\iffalse Deep Neural Networks (DNN) has significantly improved upon traditional state-of-the-art methods in image classification and speech recognition \cite{12},\cite{13},\cite{14} and recently DNN appeared promising as decoders \cite{15}. Solving problems such as image classification and object detection using deep neural networks (DNN) can be treated as problem of approximating an unknown function by DNN.\fi

It is well-known that 
%neural networks are universal approximators: 
sufficiently large  multi-layer feedforward networks can approximate any function with desired accuracy~\cite{7}. 
An important problem then is to determine the smallest neural network for a given task and accuracy.  
The standard guideline is the approximation power (variously known as expressiveness) of the network which quantifies the size of the neural network, typically in terms of depth and width, in order to approximate a class of functions within a given error. 
\iffalse 
An extensive body of literature exists on the approximation power of neural networks, see, {\it{e.g.}}, \cite{2,6,8,1,5,Pascanu+et+al-ICLR2014b}. 
With the success of deep learning in practical applications in the last decade, numerous papers recently studied the approximation power in terms of the 
network depth and width. \fi In particular, several works
provided evidence that deeper networks perform better than shallow ones, given a fixed number of hidden units \cite{delalleau2011shallow,Pascanu+et+al-ICLR2014b,bianchini2014complexity,2,3,8,1,5}.\footnote{For a nice counterexample see \cite{DBLP:conf/nips/LuPWH017}.} 

A popular activation function is the rectified linear unit (ReLU), partly  because of its low complexity when coupled with backpropagation training \cite{14}. It has, therefore, become of interest to determine the power of neural networks with ReLU's and, more generally, with 
piecewise linear activation functions. 

Determining the capacity of a neural networks with a piecewise linear activation function typically involves two steps. 
First, evaluate the number of linear pieces (or break points) that the network can produce and, second, tie this number to the approximation error. The works  \cite{Pascanu+et+al-ICLR2014b,18} recently showed that a linear increase in depth results in an exponential growth in the number of linear pieces as opposed to width which results only in a polynomial growth. Accordingly, the approximation capacity exhibits a similar tradeoff between depth and width. For related works with respect to classification error see \cite{2,3} and with respect to function approximation error see \cite{8,1,5}.

In this paper we consider general feedforward neural networks with piecewise linear activation functions and establish bounds on the size of the 
network in terms of the approximation error, the depth $d$, the width, and the dimension of the input space to approximate a given function. We first establish an improved upper bound on the number of break points that such a network can produce which is a multiplicative factor $d^d$ smaller than the currently best known from~\cite{5}. This upper bound is obtained by investigating neuron state transitions as introduced in \cite{6}.
%which consists in the study of the state of a neural network, seen as the compound of the individual states of the hidden units, as the input changes along a given direction. 
Combining this upper bound with lower bounds in terms of error and dimension, we obtain necessary conditions on the depth, width, error, and dimension for a neural network to approximate a given function. These bounds  significantly improve on the 
corresponding state-of-the-art bounds for certain classes of functions (Theorems~\ref{theorem1},\ref{theorem2} and Corollaries~\ref{weak},\ref{corollary1},\ref{corollary2}). 

The second contribution of the paper (Theorem~\ref{theorem3}) is an upper bound on the difference of two neural networks with identical weights but different activation functions. This problem is related to ``activation function simulation'' investigated in \cite{16} which leverages network topology to compensate a change in activation function. 

The paper is organized as follows. In Section~\ref{prelim} we briefly introduce the setup. In Section~\ref{mainresults} we present the main results which are then compared with the corresponding ones  in the recent literature in Section~\ref{comparison}. Finally, Section~\ref{analysis} contains the proofs.

\section{Preliminaries}\label{prelim}
 Throughout the paper $\mathcal{R}$ denotes a compact  convex set in $\mathbb{R}^n$, $n\geq 1$, and  ${\mathbb{F}}_\sigma$ denotes the set of feedforward neural networks with input $\mathcal{R}$, output $\mathbb{R}$, and  activation function $\sigma: {\mathbb{R}}\rightarrow {\mathbb{R}}$.
Feedforward here refers to the fact that the neural network contains no cycles;  connections are allowed between non-neighbouring layers. It is assumed that $\sigma$ is a piecewise linear (not necessarily continuous) function with $t\geq 1$ linear pieces. The set of all such activation functions is denoted by $\Sigma_t$.
 
 A neural network $f\in\mathbb{F}_\sigma$ consists of a set of input units ${\mathcal{I}}_f$, a set of hidden units $\cH$ that operate according to~$\sigma$, non-zero weights representing connections, and a single output unit which just weight-sums its inputs. To simplify the notation we use $f$ to represent both a neural network and the function that it represents.
 
 For instance, in the neural network shown in Fig.~\ref{fig1}, we have $\mathcal{I}_f=\{x_1,x_2,x_3\}$ and $\cH=\{u_{ij}, ~ \forall i,j\}$.

 \begin{definition}[Depth and width]\label{dpth}
Given a neural network $\cNN\in {\mathbb{F}}_\sigma $, the depth of a hidden unit $h \in \cH $, denoted as $d_f(h)$, is the length of the longest path from any $i\in {\mathcal{I}}_f$ to $h$. The depth of $f$ is 
 \begin{equation*}
d_f \defeq \max \Big \{d_f(h) \big | h \in \cH   \Big \}.
 \end{equation*}
 The set of hidden units with depth $i$ is
 \begin{equation*} 
 \cH^i \defeq\Big \{ h \in \cH  \big | d_f(h)=i   \Big\} .
 \end{equation*}
 The width of the network is
 \begin{align}\label{ell}
 \omega_f&\defeq \frac{|{\cH }| }{d_f}\defeq \frac{\sum_{i=1}^{d_f}\omega_i}{d_f} 
 \end{align}
 where $$\omega_i\defeq |\cH^i|.$$
 \end{definition}
 
 For instance, in Fig.~\ref{fig1}, the hidden unit $u_{23}$ can be reached by inputs $x_1$ and $x_3$, by following the paths $x_1\rightarrow u_{23} $, $x_3\rightarrow u_{11}\rightarrow u_{23}$, or $x_3\rightarrow u_{12}\rightarrow u_{23}$. Therefore, $d_f(u_{23})=2$. The hidden units of maximum depth are $u_{31}$, $u_{32}$, and $u_{33}$ and hence $d_f=3$, $\cH^3=\{u_{31},u_{32},u_{33} \}$ and $\omega_f=8/3$. 

%The following lemma provides a simple inequality we shall often refer to:
The following simple inequality is frequently used in the paper.
\begin{lemma}\label{bineq}
For any $t\geq 1$, $d_f\geq 1$, and $|\cH|\geq 1$
$$((t-1)\omega_f+1)^{d_f}\leq t^{|\cH|}.$$
\end{lemma}
\begin{proof}
Set $\omega_f=\frac{|{\cH }| }{d_f}$ and observe that 
$$\left((t-1)\frac{|{\cH }| }{d_f}+1\right)^{d_f}$$
is a non-decreasing function of $d_f$ and that $d_f\leq |\cH|$. 
 \end{proof}

%iffalse
  \begin{center}    
 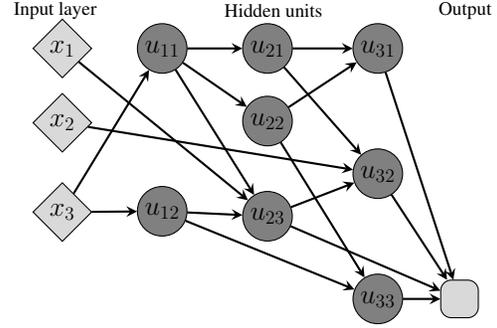
\begin{figure}[t] 
 \centering
 \begin{tikzpicture}[node distance=4cm,scale =0.5, every node/.style={scale=0.5}][h]  
 \node(in1) [input,right=2cm] {\huge$x_1$};
 \node(in2) [input, below = 0.2cm of in1] {\huge $x_2$};
 \node(in3) [input, below= 0.4cm of in2]{\huge$x_3$};
 \node(h1) [hidden, right =2cm of  in1]{\huge$u_{21}$} ;
 \node(h12) [hidden, below = 0.3cm of h1] {\huge$u_{22}$};
 \node(h13) [hidden, below= 0.6cm of h12]{\huge$u_{23}$};
 \node(h0) [hidden, right= 0.6cm of in1]{\huge$u_{11}$} ;
 \node(h02) [hidden, below = 1.5cm of h0] {\huge$u_{12}$};
 \node(h2) [hidden, right= 0.8cm of h1]{\huge$u_{31}$};
 \node(h22) [hidden, below = 1cm of h2] {\huge$u_{32}$};
 \node(h23) [hidden, below= 1cm of h22]{\huge$u_{33}$};
 \node[circle, label=left:\Large{Input layer}] (n1) at (6,1) {};
 \node[circle, label=left:\Large{Hidden units}] (n2) at (12,1) {};
 \node[circle, label=left:\Large{Output}] (n3) at (16.5,1) {};
 \node(out1) [output,right=0.5cm of h23]{};
 \draw[arrow] (h0)--(h1);
 \draw[arrow] (in2) edge (h22);
 \draw[arrow] (h0)--(h12);
 \draw[arrow] (h02)--(h23);
 \draw[arrow] (h0) edge (h13);
 \draw[arrow] (in1)--(h13);
 \draw[arrow] (h2)--(out1);
 \draw[arrow] (h22)--(out1);
 \draw[arrow] (h23)--(out1);
 \draw[arrow] (h1)--(h2);
 \draw[arrow]  (h12)--(h2);
 \draw[arrow] (in3)--(h02);
 \draw[arrow] (h02)--(h13);
 \draw[arrow] (in3)--(h0);
 \draw[arrow]  (h13) edge (out1);
 \draw[arrow]  (h1)--(h22);
 \draw[arrow]  (h13)--(h22);
 \draw[arrow]  (h12)-- (h23);
 \end{tikzpicture}
 \caption{A feedforward network $f$ with $ |{\mathcal{I}}_f|=3$ inputs, $|{\cH }| =8$ hidden units, depth $d_f=3$, 
 and width $\omega_f=8/3$. }
 \label{fig1}
 \end{figure}
 \end{center}
% \fi

   \begin{definition}[Affine $\varepsilon$-approximation] \label{def4}
Function $\cNN\in {\mathbb{F}}_\sigma $ is an affine $\varepsilon$-approximation of a function 
$g:\mathcal{R} \rightarrow \mathbb{R}$ if 
 \begin{equation*}
     \sup\limits_{\vect{x}  \in \mathcal{R}}^{} |f(\vect{x} )-g(\vect{x} )| \leq \varepsilon.   
 \end{equation*}
  \end{definition}
  
\begin{definition}[Break point]
 Given $(\vect{x},\vect{y}) \in \mathcal{R}^2$, function $f:\mathcal{R} \rightarrow \mathbb{R}$ admits a \emph{break point} at $\alpha_0\in (0,1)$ relative to the segment $[\vect{x},\vect{y}]$ if the first order derivative of $ f((1-\alpha)\vect{x} +\alpha\vect{y} )$ does not exist at $\alpha=\alpha_0$. 
  The total number of break points of $f$ on the (open) segment $]\vect{x},\vect{y}[$ is denoted by $B_{\vect{x}\rightarrow \vect{y}}(f)$. Finally, we let $\bar{B}_{\vect{x}\rightarrow \vect{y}}(f)\defeq B_{\vect{x}\rightarrow \vect{y}}(f)+1$. 
 \end{definition}
Since $f$ is piecewise linear $\bar{B}_{\vect{x}\rightarrow \vect{y}}(f)$ simply counts the number of linear pieces that $f$ produces as the input ranges from  $\vect{x}$ to $\vect{y}$.

 \iffalse
  For notational convenience let
   \begin{equation}\label{def3}
   f_{[\vect{x},\vect{y}]}(\alpha) \defeq f(\alpha\vect{x}+(1-\alpha)\vect{y}) \quad 0\leq \alpha\leq 1.
   \end{equation}
 \fi
 \section{Main Results}\label{mainresults}

Theorems~\ref{theorem1},\ref{theorem2} and Corollaries~\ref{corollary1},\ref{corollary2} provide bounds on the size of a
neural network to approximate a given function. These bounds are expressed in terms of the approximation error and width and depth of the network, but hold irrespectively of the weights. Recall that connections are allowed between non-neighboring layers.

As a notational convention we use $C^2(\mathcal R)$ to denote the set of functions ${\mathcal{R}} \rightarrow \mathbb{R}$ whose second order partial derivatives are continuous over  $ \mathring{\mathcal{R}}$ (the interior of $\mathcal{R}$).

%---recall that, by assumption, to any connection is assigned a non-zero weight.
\iffalse
 Throughout this section we consider functions $f:\mathcal{R} \rightarrow \mathbb{R}$ where $\mathcal{R}\subseteq \mathbb{R}^n $ is convex. 
 \fi

%[FIXME: to be moved where needed]

 %For special case of $1$-dimensional functions theorem $3$ changes into thi
 %Let $f$ on $[a,b]$ be two times differentiable function. Denote the number of linear pieces for $\varepsilon$-approximation of $f$ by $S(\varepsilon)$ then 
 %    \begin{equation}
     %S(\varepsilon) \geq \frac{b-a}{4\sqrt{\varepsilon}}.\Big( %\inf\limits_{x \in [a,b]}^{} \sqrt{|f''(x)|} \Big)
     %\end{equation}
 %\end{corollary}
 %[FIXME: the corollary assumes that the approximation exists. We need to provide references to justify.]

 \begin{theorem}{\label{theorem1}}
 Let $f \in {\mathbb{F}}_\sigma$, $\sigma \in \Sigma_t $, be an $\varepsilon$-approximation of a function $g\in C^2(\mathcal R)$ and let $\vect{x},\vect{y}\in \mathcal{R}$. Then, 
 \begin{align}\label{elprimero1}
     \Big((t-1) \omega_f+ 1 \Big)^{d_f} &\geq\bar{B}_{\vect{x}\rightarrow \vect{y}}(f) \\
     &\geq   \frac{||\vect{x} -\vect{y} ||_2}{4\sqrt{\varepsilon}}\cdot \Psi(g,\vect{x},\vect{y}),    \label{elprimero2}
 \end{align}
 where  
\begin{align}\label{psii}
 %\omega&\defeq \frac{|{\cH }| }{d_f} \\
 \Psi(g,\vect{x},\vect{y})  &\defeq \sqrt{\inf \limits_{0 \leq \alpha \leq 1}^{}\Big ( {\max \big\{0, \gamma(\alpha)\delta(\alpha) \big\} }\Big )},\\
 \gamma(\alpha)&\defeq\min\big\{|\alpha_1(\alpha)|,|\alpha_2(\alpha)|\big\},\notag\\
 \delta(\alpha)&\defeq \mathrm{sign}\big(\alpha_1(\alpha)\alpha_2(\alpha)\big),\notag
 \end{align}
 and where $\alpha_1(\alpha)$ and $\alpha_2(\alpha)$ are the largest and smallest eigenvalues of the hessian matrix $\nabla^2 g\big( (1-\alpha)\vect{x} + \alpha \vect{y} \big)$, respectively.
\end{theorem}
Maximizing the right-hand side of \eqref{elprimero2} over $\vect{x},\vect{y}$ and using Lemma~\ref{bineq} we obtain:
 \begin{corollary}\label{weak}
Under the assumptions of Theorem~\ref{theorem1} we have  
 \begin{equation*}
  \begin{aligned}
 |{\cH }|  \geq \log \limits_{t}^{}  \Biggl(  \sup \limits_{(\vect{x},\vect{y}) \in \mathcal{R}^2}^{} \Big \{ \frac{||\vect{x} -\vect{y} ||_2}{4\sqrt{\varepsilon}}\cdot \Psi(g, \vect{x},\vect{y})   \Big\} \Biggl).
 \end{aligned}
 \end{equation*}
 \end{corollary}

 %\subsection{Comparisons with other lower bounds}
 
 A function $g:\mathcal R\rightarrow \mathbb R$ that is twice differentiable is said to be strongly convex with parameter $\mu$ if 
 $\nabla^2 g(\vect{x})\succeq \mu I$ for all $\vect{x}\in\mathring{\mathcal R}$.
 
 \begin{corollary}{ \label{corollary1}}
 %Let $\cNN \in {\mathbb{F}}_\sigma$, $\sigma \in \Sigma_t $, be an $\varepsilon$-approximation of a 
%strongly convex function  $g: \mathcal{R} \rightarrow \mathbb{R} $ with parameter $\mu>0$. Then, 
 Let $f \in {\mathbb{F}}_\sigma$, $\sigma \in \Sigma_t $, be an $\varepsilon$-approximation of a function
 $g\in C^2(\mathcal R)$ that is strongly convex  with parameter $\mu>0$. Then, 
 \begin{equation*}
 |{\cH }|  \geq \frac{1}{2}\log_{t}^{} \Big( \frac{\mu \cdot(\mathrm{diam(\mathcal{R})})^2}{16{\varepsilon}}     \Big),
 \end{equation*}
  where 
  \begin{equation*}
  \mathrm{diam}(\mathcal{R})\defeq \sup \limits_{(\vect{x},\vect{y})\in \mathcal{R} }^{} ||\vect{x} - \vect{y}||_2. 
  \end{equation*}
 \end{corollary}

 \begin{proof}
 By strong convexity $\Psi(g, \vect{x},\vect{y})  \geq \sqrt{\mu}$. The result then follows from Theorem~\ref{theorem1} and Lemma~\ref{bineq}.
 \end{proof}
 
As an example, consider $g(\vect{x})=\vect{x}\cdot\vect{x}$ over $[0,1]^n$. The Hessian matrix is $2I_{n \times n}$ and from Corollary~\ref{corollary1} we get
   \begin{equation*} 
   |{\cH }|  \geq  \log_2 \Big( \sqrt{ \frac{n}{8 \varepsilon} }  \Big).
   \end{equation*}

 \begin{corollary} \label{corollary2}
Let $\mathcal R=[0,1]^n$. Let $f \in {\mathbb{F}}_\sigma$, $\sigma \in \Sigma_2 $,\footnote{Recall that $\Sigma_2$ includes ReLU's.} be an $\varepsilon$-approximation of a function $g\in  C^2(\mathcal R)$ such that  $\nabla g(x)\succ 0$ for any $x\in  \mathring{\mathcal{R}}$. Then,
 \begin{equation}
 |\cH| \geq q(g) d_f\varepsilon^ {-\frac{1}{2d_f}}
 \end{equation}
 where $q(g)>0$ is a constant that only depends on $g$.
 \end{corollary}
 \begin{proof}[Proof of Corollary~\ref{corollary2}]
 From Theorem \ref{theorem1} we get 
  \begin{align*}
  &\Big(\frac{\cH}{d_f}+ 1 \Big)^{d_f} \geq \frac{c(g)}{\sqrt{\varepsilon}} ,
\end{align*}
 where $c(g)>0$ is some strictly positive constant, since the Hessian of $g$ is positive definite everywhere over $ \mathring{\mathcal{R}}$. Since ${\cH}/{d_f}\geq 1$ the above inequality implies
\begin{align*}
\Big(2\frac{|\cH|}{d_f}\Big)^{d_f}  \geq \frac{c}{\sqrt{\varepsilon}}. 
 \end{align*}
 Since $\frac{1}{2}c^{\frac{1}{d_f}}\geq q$  where  $q=\frac{1}{2}\min(c,1)$, the above inequality 
 yields the desired result.
  \end{proof}

 \begin{theorem} \label{theorem2}
Let $\mathcal R =[0,1]^n$.  Let $\cNN \in {\mathbb{F}}_\sigma$, $\sigma \in \Sigma_t $, be an $\varepsilon$-approximation of a function  $g:\mathcal R  \rightarrow \mathbb{R}$ such that $|D^{J}(g)(\vect{x})| \leq \delta$ for any $\vect{x} \in [0,1]^n  $ and any 
 multi-index\footnote{\textit{E.g.}, for $J=(2,1)$ we have $D^J(g(x_1,x_2))=\frac{\partial^3 g }{\partial^2 x_1 \partial x_2} $.}
 $J $ such that $|J|=3$. Then, 
 \begin{equation}
  \big((t-1) \omega_f+ 1 \big)^{d_f} \geq \sqrt{\frac{\Big ( \max \limits_{ \vect{x} \in [0,1] ^ n } ^ {} \big| {\Delta(g)(\vect{x})}\big|  n^{-1}  - \delta n^\frac{3}{2} \Big)^+}{16\varepsilon}},  
 \end{equation}
 where 
  \begin{equation}
 \Delta (g)(\vect x) = \sum\limits_{k=1}^n \frac{d^2 g}{d x_k^2},
 \label{eq:laplacian}
 \end{equation}
is the Laplacian of $g$ and where $a^+=\max(a,0)$.
 
 \end{theorem}
 
 For instance, approximating $$g(x_1,x_2)=10x_1^2+x_1^2x_2^2+10x_2^2$$ over $[0,1]^2$ requires $\log_t \Big( \frac{0.82}{\sqrt{\varepsilon}}   \Big) $ hidden units---combine Theorem~\ref{theorem2} with Lemma~\ref{bineq}.

Whether it is Theorem~\ref{theorem1} or Theorem~\ref{theorem2} which provides a better approximation bound depends on $g$. For instance, for $g_1(x_1,x_2)= 20x_1^2-2x_2^2+x_1^2x_2^2$ Theorem~\ref{theorem1} gives a trivial (zero) lower bound since the two eigenvalues of the Hessian matrix $\nabla^2(g_1)$ have always different signs. Theorem~\ref{theorem2} instead gives $\frac{0.737}{\sqrt{\varepsilon}}$. On the other hand, for $g_2(x_1,x_2)=10x_1^2+10x_2^2+x_1^2x_2^2$ Theorem~\ref{theorem1} gives $\frac{1.37}{\sqrt{\varepsilon}}$ as lower bound while Theorem~\ref{theorem2} gives $\frac{0.82}{\sqrt{\varepsilon}}$.

 \iffalse
 Given $\cNN \in \Pi$ we denote by $f^{\cNN}_{\sigma}(\vect{x}) $ the function given by neural network $\cNN$ and activation function  $\sigma$.
 \fi
 
 The next theorem quantifies the effect of a change of activation function on the output of the neural network. Here, the activation functions need not be
 piece-wise affine.
 
 \begin{theorem}{\label{theorem3}}
Let $f_1\in\mathbb{F}_{\sigma_1}$ and $f_2\in\mathbb{F}_{\sigma_2}$ be two neural networks with identical architectures and weights.
%, but with activation functions $\sigma_{1}$ and $\sigma_{2}$, respectively. 
Suppose that $\sigma_1$ is a $\delta$-Lipschitz continuous function and suppose that the weights belong to some bounded interval $[-A,+A]$, $A>0$. Then,  
 \begin{equation}\label{mism}
 ||f_1- f_2||_{\infty} \leq \frac{||\sigma_1 -\sigma_2||_{\infty}}{\delta} \Bigg( \Big(\delta\cdot A \cdot \omega_f +1\Big)^{d_f} -1 \Bigg ).
 \end{equation}
 \end{theorem}
A slightly weaker version of \eqref{mism} is  
\begin{align*}
 ||f_1- f_2||_{\infty} \leq \frac{||\sigma_1 -\sigma_2||_{\infty}}{L} \Bigg( \Big(L^2\cdot \omega_f +1\Big)^{d_f} -1 \Bigg ),
\end{align*}
where $L=\max\{A,\delta\}$ denotes the  \emph{Lipschitz-bound} defined in ~\cite{16}. 

As an illustration of Theorem~\ref{theorem3} consider a feedforward neural network $f_1$ with $100$ hidden units, a maximum depth of $5$, and the \emph{sigmoid} as activation function. Suppose the weights belong to interval $[-1,1]$. Replacing the sigmoid with a $32$-bit quantized function results in an error of at most $0.0001$---which can readily be obtained from Theorem~\ref{theorem3} with $\delta=\frac{1}{4}, A=1, ||\sigma_1-\sigma_2||_\infty=2^{-32}$.  
 
\section{Comparison with Previous Works}\label{comparison}

Consider first the inequality \eqref{elprimero1}. Restricting attention to neural networks 
%$f\in \mathbb{F}_{\sigma}$, $\sigma\in\Sigma_t$, 
with $d$ hidden layers, at most $\omega$ units per layer, and where connections are allowed only between neighbouring layers, this inequality gives
\begin{equation} {\label{rel1}}
\bar{B}_{\vect{x}\rightarrow \vect{y}}(f)\leq \Big((t-1) \omega+ 1 \Big)^{d}.
\end{equation} 
 This is to be compared with the previously best known bound (Lemma $3.2$ in  \cite{3}) 
  $$2( 2 (t-1) \omega )^d   $$
  which is larger by a multiplicative factor that is exponential in $d$ whenever $\omega>1 $, $t\geq 2$. For $n=1$, Lemma~2.1 in \cite{2} gives $(t\omega)^d$ which still differs from \eqref{rel1} by a multiplicative factor that is exponential is $d$ for $\omega>1 $, $t\geq 2$.
 
 For general feedforward neural networks the previously best known bound (see Lemma 4 of \cite{5}) was
  $$\bar{B}_{\vect{x}\rightarrow \vect{y}}(f) \leq \Big(t\cdot \omega \cdot d_f \Big)^{d_f}$$ 
  which is a multiplicative factor ${d_f}^{d_f}$ larger than \eqref{elprimero1}.

 %Lemma 4 of \cite{5} obtained such an upper bound for general feedforward neural networks.
 % It suggests $(tmd)^d$ which is so much larger than $\big((t-1)l+1\big)^d $.

Now consider the approximation power of neural networks in terms of number of hidden units required to approximate a given function within a given error. Theorem~11 in \cite{1} states that to approximate a function $[0,1]^n \rightarrow \mathbb{R}$, assumed to be differentiable and strongly convex with parameter $\mu$, with a neural network $f$ requires $$|\cH|\geq \frac{1}{2}\log_2 \big( \frac{\mu}{16\varepsilon} \big),$$ regardless of the dimension $n$. Corollary \ref{corollary1} improves this bound to $$\frac{1}{2}\log_2 \big( \frac{\mu \cdot n}{16\varepsilon} \big) $$ which incorporates dimension as well---albeit the dependency on dimension is arguably small. 
 
\iffalse [FIXME: HERE DISCUSS TH1 VS TH2, SAY THAT TH2 BOUND EASIER TO EVALUATE, ONLY A LAPLACIAN] \fi
 
 \begin{savenotes}
     \begin{table}[t]
     \caption{Bounds comparisons}
     \label{comparisson}
     \vskip 0.15in
     \begin{center}\label{table1}
     \begin{small}
     %\begin{sc}
     \begin{tabular}{lcc}
     \toprule
     & Previous & This paper  \\
     \midrule
     Regular:   & \cite{3} &(Theorem~\ref{theorem1})  \\$\bar{B}_{\vect{x}\rightarrow \vect{y}}(f)\leq $  &  $2( 2 (t-1) \omega )^d  $ & $ \Big((t-1) \omega+1 \Big)^{d}$   \\ 
     \midrule
    General:   & \cite{5} &(Theorem~\ref{theorem1})  \\$\bar{B}_{\vect{x}\rightarrow \vect{y}}(f)\leq  $ &$\Big(t\cdot \omega \cdot d_f \Big)^{d_f}$ & $ \Big((t-1) \omega_f+1 \Big)^{d_f}$   \\ 
     \midrule
      $g\in C^2([0,1]^n)$  &  &   \\ over $\mu$-convex    &\cite{1} &  (Corollary \ref{corollary1})  \\ 
    $|\cH|\geq $  &  $\frac{1}{2}\log_2 \big( \frac{\mu}{16\varepsilon} \big)$ & $\frac{1}{2}\log_2 \big( \frac{\mu \cdot n}{16\varepsilon} \big) $ \\ 
     \midrule
      $g\in C^2([0,1]^n)$ & & \\
    $\text{Hess}(g){\succ} 0$, $\Sigma_2$ &\cite{5}  &(Corollary \ref{corollary2})  \\
     $|\cH|\geq$ &  $q_1\varepsilon^{\frac{-1}{2d_f}}$    & $d_f q_2 \varepsilon^{\frac{-1}{2d_f}}$  \\
     \iffalse \midrule
    $ \mathcal{W}^{m,\infty}\big([0,1]^n\big)$ &    $\Omega(\varepsilon^{-\frac{n}{2m}})$ & -\\\fi
     \bottomrule
     \end{tabular}
     %\end{sc}
     \end{small}
     \end{center}
     \vskip -0.1in
     \end{table}
     \end{savenotes}

 Corollary \ref{corollary2} provides a lower bound for ReLU types of networks in terms of the error, the depth, and a constant term which only depends on $g$. This bound can be compared with the bound of Theorem 6 in \cite{5} which is of order $\epsilon^{-\frac{1}{2d_f}}$.\footnote{Theorem~6 of~\cite{5} provides a bound of the form $q\epsilon^{-\frac{1}{2d_f}}$ where $q$ is a constant that depends on both $g$ and $d_f$. However, a close inspection of the proof of this theorem reveals that $q$ depends only on $g$. } Hence, Corollary~\ref{corollary2} provides a linear (in $d_f$) improvement which is particularly relevant in the deep regime where $d_f=\Omega(\log(1/\varepsilon))$. Table \ref{table1} summarizes the above discussion.

To the best of our knowledge Theorem~\ref{theorem3} is the first result to bound the effect of a change in the activation function for  given network topology and weights. Noteworthy perhaps, this bound is essentially universal in the weights since it only depends on their range.

Finally,  compared to the cited papers it should perhaps be stressed that the proofs here (see next section) are relatively elementary---{\it{e.g.}}, they do not hinge on VC dimension analysis---and hold true
for general feedforward networks.

 %\appendices

 \section{Analysis}\label{analysis}
We first establish a few lemmas to prove Proposition~\ref{proposition1} which will provide an upper bound on the number of break points. Then we establish Propositions~\ref{proposition2} and~\ref{proposition3} which will give lower bounds on the number of break points in terms of the approximation error. Combining these propositions will give Theorems~\ref{theorem1} and~\ref{theorem2}. Finally, we prove Theorem~\ref{theorem3}.
 \begin{definition}[Intermediate set of units]
Given $\cNN\in {\mathbb{F}}_\sigma $ and $\cU \subseteq \cH  $ we define the set of hidden units that lie on a path between the input and $\cU$ as 
 \begin{equation*}
    \mathrm{in}(\cU)\defeq \Big \{ v \in \cH  \backslash \cU | \exists i \in {\mathcal{I}}_f, u \in \cU \: \mathrm{s.t.} \: v \in (i \rightarrow u) \Big \} 
 \end{equation*}
 where $(i\rightarrow u)$ denotes the set of intermediate hidden nodes on the path from $i$ to $u$.   
 \end{definition}
 For instance, in Fig.~\ref{fig1} we have $$\mathrm{in}(\{u_{32} \})=\{u_{11},u_{12},u_{21},u_{23} \}.$$ 
 
 The following lemma follows from the above definition.
 %\footnote{With a slight abuse of notation $\mathrm{in}(u)$ denotes $\mathrm{in}(\{u\})$.}
 \begin{lemma}\label{inin}
 Given $\cU\subseteq \cH $ we have 
  $$\mathrm{in}(\mathrm{in}(\mathcal U)=\emptyset$$ and
 \begin{equation*}
   \mathrm{in}(u) \subseteq  (\cU \cup  \mathrm{in}(\cU)) 
 \end{equation*}
 for any  $ u \in \cU $. 
 \end{lemma}
\iffalse \subsection*{Linear piece activation functions}
 In this section we restrict ourselves to the set of real $t\geq 1$ linear piece activation functions, which we denote as~$\Sigma_{t}$.  For instance,  \emph{Rectifier Linear Units(ReLU)} or \emph{binary step} activations functions belong to $\Sigma_2$. \fi
 \begin{definition}[State]\label{def2}
 Any $\sigma\in \Sigma_{t} $ partitions the real line (its input) into $t$ intervals $I_1,I_2,...,I_t$ such that on each of these intervals $\sigma$ is affine.
  The state of a unit with activation function $\sigma$ is defined to be $s\in \{1,2,\ldots,t\}$ if its input belongs to $I_s$. By extension, the state of $\cU \subseteq \cH $ is defined to be the vector of length $|\cU|$ whose components are the state of each unit in $\cU$.  
 \end{definition}
 The following definition is inspired by the notion of pattern transition introduced in~\cite{6}:
 \begin{definition}[Transition]
 Let $f\in \mathbb F_{\sigma}$,  $\cU \subseteq \cH $ and  $\vect{x},\vect{y} \in \mathcal{R}$.
 Let $\vect z_\alpha=(1-\alpha)\vect{x}+\alpha\vect{y}$  be a parametrization of 
 the line segment $[\vect x,\vect y]$ as $\alpha$ goes from $0$ to $1$. 
 We say that the state
 of $\cU$ experiences a transition at point $\vect z_{\alpha^*}$ for some $\alpha^*\in (0,1]$ if the state vector of $\cU$ changes at $\vect z_{\alpha^*}$ while the state vector of $\text{in}(\cU)$ does not
 change at $\vect z_{\alpha^*}$. The number of state transitions of $\cU$ on the segment 
 $[\vect x,\vect y]$, denoted by 
 $N_{\vect{x} \rightarrow \vect{y}} (\cU)$, is defined to be the number of state transitions of $\cU$ as the input changes from $\vect{x}$ to $\vect{y}$ on $\vect z_\alpha$. 
 If $\mathrm{in}(\cU)=\emptyset$, then $N_{\vect{x} \rightarrow \vect{y}} (\cU)$ is defined to be the number of state transitions of $\cU$ as the input
 changes from $\vect x$ to $\vect y$. 
  \end{definition}
 
 Note that  if the state vectors of both $\cU$ and $\text{in}(\cU)$ change at  $\alpha$,  $N_{\vect{x} \rightarrow \vect{y}} (\cU)$ does not
 change at that  $\alpha$. For example,  consider the neural network $f$ in Fig.~\ref{fig1}. Suppose that $\cU=\{u_{11}, u_{12}\}$ and suppose that the state of 
 $u_{11}$ and $u_{12}$ changes exactly once along segment $z_\alpha$ for some $\vect x$ and $\vect y$, 
 respectively at $\alpha_1$ and $\alpha_2$. Then $N_{\vect{x} \rightarrow \vect{y}} (\{u_{11}\})=1$ and 
 $N_{\vect{x} \rightarrow \vect{y}} (\{u_{12}\})=1$. If $\alpha_1=\alpha_2$,  $N_{\vect{x} \rightarrow \vect{y}} (\cU)=1$, otherwise 
 $N_{\vect{x} \rightarrow \vect{y}} (\cU)=2$. If $\cU'=\{u_{21}, u_{22}, u_{23}\}$, 
 and  the state of each of 
 $u_{21}$, $u_{22}$ and $u_{23}$ changes exactly once at either $\alpha_1$ or $\alpha_2$, then $N_{\vect{x} \rightarrow \vect{y}} (\cU')=0$ since 
 the state vector of $\mathrm{in}(\cU')=\cU$ has also changed at both $\alpha_1$ and $\alpha_2$.

 \begin{lemma}\label{lemma2}
Given $\cNN\in {\mathbb{F}}_\sigma $ and $\cU_1,\cU_2\subseteq \cH$ such that $\mathrm{in}(\cU_2)=\emptyset$ and $ \mathrm{in}(\cU_1) \subseteq \cU_2 $, we have
 \begin{equation*}
  N_{\vect{x} \rightarrow \vect{y}} \Big( \cU_1 \cup \cU_2\Big) \leq  N_{\vect{x} \rightarrow \vect{y}} \Big( \cU_1\Big) +  N_{\vect{x} \rightarrow \vect{y}} \Big( \cU_2\Big).
 \end{equation*}
 \end{lemma}
 \begin{proof}
 Suppose $N_{\vect{x} \rightarrow \vect{y}} \Big( \cU_1 \cup \cU_2\Big)$ increases by one at $\alpha=\alpha^*$. If $\cU_2$ undergoes a state transition at $\alpha^*$ then, because $\mathrm{in}(\cU_2)=\emptyset$, we have that $N_{\vect{x} \rightarrow \vect{y}} \Big( \cU_2\Big) $
 also increases by one at $\alpha^*$. Instead, if no state change happens in $\cU_2$ at $\alpha^*$ then, due to the state change of $ \cU_1 \cup \cU_2$ at $\alpha^*$, the state of $\cU_1$ must change as well at $\alpha^*$. Since $\mathrm{in}(\cU_1) \subseteq \cU_2$ and no change in the state of $\cU_2$ is observed at $\alpha^*$ we have that $N_{\vect{x} \rightarrow \vect{y}} \Big( \cU_1\Big)$ necessarily increases by one at $\alpha^{*}$. 
 \end{proof}
 \begin{lemma} \label{lemma3}
Given $\cNN\in {\mathbb{F}}_\sigma $ and $\cU_1,\cU_2\subseteq\cH$ such that $\cU_1 \subseteq \cU_2$ and $\mathrm{in}(\cU_2)=\emptyset$ we have
 \begin{equation*}
 N_{\vect{x} \rightarrow \vect{y}} \Big( \cU_1\Big) \leq N_{\vect{x} \rightarrow \vect{y}} \Big( \cU_2\Big).
 \end{equation*} 
 \end{lemma}
 \begin{proof}
 Suppose $N_{\vect{x} \rightarrow \vect{y}} \Big( \cU_1\Big)$ increases by one at $\alpha^*$. Since $\cU_1 \subseteq \cU_2$ the state of $\cU_2$ changes as well at $\alpha^*$. Since $\mathrm{in}(\cU_2)=\emptyset$ we deduce that $N_{\vect{x} \rightarrow \vect{y}} \Big( \cU_2\Big)$ increases at $\alpha^*$ by one, thereby concluding the proof.
 \end{proof}

\begin{lemma}
\label{lemma4}
Given $\cNN\in {\mathbb{F}}_\sigma $, for any $\cU\subseteq \cH $ we have
 \begin{equation*} 
 N_{\vect{x} \rightarrow \vect{y}} (\cU) \leq \sum \limits_{u \in \cU}^{} N_{\vect{x} \rightarrow \vect{y}}(u).
 \end{equation*}
 \end{lemma}
 \begin{proof} 
 Suppose that $N_{\vect{x} \rightarrow \vect{y}} (\cU)$ increases by one at $\alpha^{*}$.
 Let $\cV \subseteq  \cU$ be the set of units that experience a transition at  $\alpha^{*}$. Since we have a transition in the state of $\cU$ at $\alpha^{*}$ we have $\cV \neq \emptyset$. Now, because the neural network is cycle-free,\footnote{Recall that throughout the paper neural networks are feedforward.} there exists some $v \in \cV$ such that $\mathrm{in}(v) \cap \cV = \emptyset$.
 We claim that the state of $\mathrm{in}(v)$ has not changed at $\alpha^{*}$. 
 To prove this note that by Lemma~\ref{inin} we have $\mathrm{in}(v) \subseteq \mathrm{in}(\cU) \cup \cU$ and since $\mathrm{in}(v) \cap \cV = \emptyset$ we deduce that
 $\mathrm{in}(v) \subseteq (\mathrm{in}(\cU) \cup \cU\backslash \cV  ).$
 On the other hand  neither $\cU\backslash \cV$ nor $\mathrm{in}(\cU)$ has a transition at $\alpha^{*}$. This implies that $\mathrm{in}(v)$ has  no transition at $\alpha^{*}$ and therefore $N_{\vect{x} \rightarrow \vect{y}}(v)$ increases by one at $\alpha^{*}$. This concludes the proof since $v\in \cU$. 
 \end{proof}

 \begin{lemma}\label{lemma5}
Given $\cNN\in {\mathbb{F}}_\sigma $, for any $u\in \cH $ we have  
 $$  N_{\vect{x} \rightarrow \vect{y}} (u) \leq (t-1) \Big(  N_{\vect{x} \rightarrow \vect{y}} (\mathrm{in}(u))+1 \Big).$$

 \end{lemma}

 \begin{proof} 
 To establish the lemma we show that between transitions of $\mathrm{in}(u)$ there are at most $t-1$ transitions of $u$.
 
 Suppose, by way of contradiction, that at least $t$ transitions in the state of $u$ happen while $\mathrm{in}(u)$ experiences no change. Then there exists an increasing sequence of real numbers $\alpha_1,...,\alpha_{t+1}$ from interval $[0,1]$ and an increasing set of integers $k_1,k_2,...,k_{t+1}$ from $S=\{1,2,...,t\}$, with $k_i\ne k_{i+1}$, such that for particular $\vect{w}\in \mathbb{R}^n$ and $b\in \mathbb{R} $ we have
 \begin{equation*}
 \begin{aligned}
 &\vect{x_i} \defeq (1-\alpha_i)\vect{x}+\alpha_i\vect{y}\\
 &\vect{w}\cdot\vect{x_i}+ b \in I_{k_i} \\
 \end{aligned}
 \end{equation*}  
 where $I_i$ is defined in Definition~\ref{def2}.
 Since $|S|=t$ there exists $i<j$ such that $k_i=k_j$. Now since  $k_i \neq k_{i+1}$ we deduce that $j \neq i+1$ and therefore $j > i+1$. But $\vect{w}\cdot\vect{x_{i+1}}+b$ lies between $\vect{w}\cdot\vect{x_i}+b$ and $\vect{w}\cdot\vect{x_j}+b$ since the sequence $\alpha_1,\alpha_2,...,\alpha_{t+1}$ is increasing.  Since $\vect{w}\cdot\vect{x_j}+b$ and $\vect{w}\cdot\vect{x_i}+b$ belong to $I_{k_i}$, by the connectedness property of the set $I_i$ we deduce that that $\vect{w}\cdot\vect{x_{i+1}}+b \in I_i$. Therefore, we get $k_i=k_{i+1}=k_j$, a contradiction. 
 %As $\sigma$ consists of $t$ pieces and we had $t$ transitions so there should exist $\alpha_1 < \alpha_2 < \alpha_3$ such that if $\vect{x_i}=(1-\alpha_i)\vect{x}+\alpha_i\vect{y}$ then    
  %Assume the state of $u$ at $t^{*}_1,t^{*}_2,...,t^{*}_t$ has changed and meanwhile no transition in the state of $\mathrm{in}(u)$ has been reported. There should exist $0 \leq t_1<t_2<...<t_ \leq 1$ such that $t^{*}_1 \in [t_1,t_2]$ and $t^{*}_2 \in [t_2,t_3]$ and for specific values of $\vect{w},b$  following relations hold
 %\begin{equation}
     %( \vect{w}.\vect{x}_1+b)(\vect{w}.\vect{x}_2+b)<0 ,\quad %(\vect{w}.\vect{x}_2+b)(\vect{w}.\vect{x}_3+b)<0
 %\end{equation}
 %where 
 %$b$ is the bias value of unit $u$ and
 %$\vect{x}_i=(1-t_i)\vect{x}+t_i\vect{x}$ for $i=1,2,3$. 
 %$\vect{x}_1,\vect{x}_2,\vect{x}_3$ are collinear so there exists $\alpha \in [0,1]$ such that $\vect{x}_2=\alpha \vect{x}_1 +(1-\alpha)\vect{x}_3$.
 %\begin{equation}
 %\begin{aligned}
 %&\Rightarrow 0 \leq (\vect{w}.\vect{x}_2+b)^2\\
 %&=\alpha(\vect{w}.\vect{x}_1+b)(\vect{w}.\vect{x}_2+b)+(1-\alpha)(\vect{w}.\vect{x}_3+b)(\vect{w}.\vect{x}_2+b)\\
 %& < 0
 %\end{aligned}
 %\end{equation}
 \end{proof}
  Since a break point of $\cNN\in {\mathbb{F}}_\sigma$ necessarily implies a change in the state of the units we get: 
 
  \begin{lemma}\label{lemma6} Given $(\vect{x},\vect{y}) \in \mathcal{R}^2$ and $\cNN\in {\mathbb{F}}_\sigma$ we have
  \begin{equation*}
 B_ {\vect{x}\rightarrow \vect{y}}(\cNN) \leq N_{\vect{x}\rightarrow \vect{y}}(\cH).
  \end{equation*}
  \end{lemma}
  
 \iffalse
 The first result provides an upper bound on the maximum number of break points that can be produced by  a general feedforward neural networks with a given depth and a given number of hidden units:

The second set of results---Theorems~\ref{theorem1},\ref{theorem2} and their corollaries---relates the number of change points of a neural network to its approximation error of a given function.
\fi
\iffalse
 Throughout this section we consider functions $f:\mathcal{R} \rightarrow \mathbb{R}$ where $\mathcal{R}\subseteq \mathbb{R}^n $ is convex. 
 \fi

\iffalse  \textcolor{red}{Now we are ready to express propositions ~\ref{proposition1},~\ref{proposition2},~\ref{proposition3} with their proofs, next we will use these propositions to provide proofs for theorems in the main section}     
  
----

 The first result provides an upper bound on the maximum number of break points that can be produced by a general feedforward neural networks \textcolor{red}{while moving along an arbitrary line in the input domain} with a given depth and a given number of hidden units, this bound is presented in proposition \ref{proposition1}.
 
\textcolor{red}{The second set of results provide a lower bound on the number of affine pieces that are required for affine $\varepsilon-$approximation of a given function $g$. These bounds can be found in propositions \ref{proposition2},\ref{proposition3}.}

-----  
  \fi
  
Propositions~\ref{proposition1} and \ref{proposition2} establish inequalities \eqref{elprimero1} and \eqref{elprimero2} of Theorem~\ref{theorem1}.
 \begin{proposition}\label{proposition1}
  Given $\cNN \in {\mathbb{F}}_\sigma$, $\sigma \in \Sigma_t$, we have
  \begin{equation}\label{upb}
  B_{\vect{x} \rightarrow \vect{y}}(\cNN) \leq 
  \Big(\big(t-1 \big)\omega_f+1\Big)^{d_f} - 1.
  \end{equation}
  %where \begin{align}\label{ell}
  %\omega\defeq \frac{|{\cH }| }{d_f}.
  %\end{align}
  \end{proposition}

 %For special case of $1$-dimensional functions theorem $3$ changes into thi
 %Let $f$ on $[a,b]$ be two times differentiable function. Denote the number of linear pieces for $\varepsilon$-approximation of $f$ by $S(\varepsilon)$ then 
 %    \begin{equation}
     %S(\varepsilon) \geq \frac{b-a}{4\sqrt{\varepsilon}}.\Big( %\inf\limits_{x \in [a,b]}^{} \sqrt{|f''(x)|} \Big)
     %\end{equation}
 %\end{corollary}
 %[FIXME: the corollary assumes that the approximation exists. We need to provide references to justify.]

% \subsection*{Proof of Proposition~\ref{proposition1}}
\begin{proof}[Proof of Proposition~\ref{proposition1}] 
 Fix $\cNN \in {\mathbb{F}}_\sigma$ where $\sigma \in \Sigma_t$. Referring to  Definition~\ref{dpth}, consider the partition $$\cup_{i=1}^d\cH^i$$ of $\cH$ according to unit depth where $d=d_f$. 
 
 Fix $u \in \cH^{i+1}$, $0 \leq i < d$. From the definitions of $\mathrm{in}(u)$ and $\cH^i$ we get
 \begin{align} \label{eq1}
 &\mathrm{in}(u) \subseteq  \bigcup \limits_{j=1}^{i}\cH^{j} \\
 & \mathrm{in}\Big(\cH^{i+1} \Big) \subseteq  \bigcup\limits_{j=1}^{i} \cH^{j}\nonumber \\
 &\mathrm{in}\Big(\bigcup\limits_{j=1}^{i} \cH^{j} \Big)=\emptyset .
\nonumber 
\end{align}
 Applying Lemma~\ref{lemma2}  with $\cU_1=\cH^{i+1}$ and $\cU_2=\bigcup\limits_{j=1}^{i} \cH^{j} $ we get
 \begin{equation*}
 \begin{aligned}
     &N_{\vect{x} \rightarrow \vect{y}}(\bigcup \limits_{j=1}^{i+1}\cH^{j}) \leq
     N_{\vect{x} \rightarrow \vect{y}}(\bigcup \limits_{j=1}^{i}\cH^{j})  +  N_{\vect{x} \rightarrow \vect{y}}(\cH^{i+1}).
 \end{aligned}
 \end{equation*}
 From Lemma \ref{lemma4}
 \begin{equation*}
 \begin{aligned}
     &N_{\vect{x} \rightarrow \vect{y}}(\bigcup \limits_{j=1}^{i+1}\cH^{j}) \leq 
     N_{\vect{x} \rightarrow \vect{y}}(\bigcup \limits_{j=1}^{i}\cH^{j})+ \sum \limits_{u \in \cH^{i+1}}^{} N_{\vect{x} \rightarrow \vect{y}}(u)
 \end{aligned}
 \end{equation*}
 and applying Lemma~\ref{lemma5} to the previous inequality
 \begin{equation*}
      \begin{aligned}
       N_{\vect{x} \rightarrow \vect{y}}(\bigcup \limits_{j=1}^{i+1}\cH^{j}) &\leq
     N_{\vect{x} \rightarrow \vect{y}}(\bigcup \limits_{j=1}^{i}\cH^{j})\\
     &+ \sum \limits_{u \in \cH^{i+1}}^{} \big( t-1 \big )\Big(N_{\vect{x} \rightarrow \vect{y}}\big(\mathrm{in}(u)\big)+1\Big).
 \end{aligned}
 \end{equation*}
 Then, using \eqref{eq1} and Lemma~\ref{lemma3} we  get
   \begin{align}{\label{eq2}}
   &N_{\vect{x} \rightarrow \vect{y}}(\bigcup \limits_{j=1}^{i+1}\cH^{j})\nonumber\\
   & \leq
     N_{\vect{x} \rightarrow \vect{y}}(\bigcup \limits_{j=1}^{i}\cH^{j})
     + \sum \limits_{u \in \cH^{i+1}}^{} \big( t-1 \big ) \Big ( N_{\vect{x} \rightarrow \vect{y}}\big(\bigcup \limits_{j=1}^{i}\cH^{j}\big)+1\Big)\nonumber\\
     &= \big(\omega_{i+1}(t-1)+1\big)N_{\vect{x} \rightarrow \vect{y}}(\bigcup \limits_{j=1}^{i}\cH^{j})+\omega_{i+1}(t-1).
     % &\leq (\omega_{i+1}+1)\Big( T_{\vect{x} \rightarrow \vect{y}}(\bigcup \limits_{j=1}^{i}\cH^{j})+1\Big).
\end{align}
 For $u \in \cH^1$ we have $\mathrm{in}(u) = \emptyset$ and according to Lemma \ref{lemma5} we deduce that $N_{\vect{x} \rightarrow \vect{y}}(\cH^1) \leq (t-1)\omega_1$. With this initial condition and the recursive relation in \eqref{eq2} we get 
 \begin{align*}
        &N_{\vect{x} \rightarrow \vect{y}}(\bigcup \limits_{j=1}^{d}\cH^{j}) \nonumber\\ 
         & \leq \sum \limits_{j=1}^{d} \Bigg( \prod \limits_{1 \leq \alpha_1 < \alpha_2<...<\alpha_j\leq d}^{} \omega_{\alpha_1}\omega_{\alpha_2}\cdots\omega_{\alpha_j} \big(t-1 \big)^j  \Bigg)\nonumber\\
         & \leq \sum \limits_{j=1}^{d} \Big( {d \choose j}  \big(\omega_f(t-1)\big)^j    \Big) = \Big(\omega_f (t-1)+1\Big)^d-1 \nonumber 
         \end{align*} 
with $\omega_f$ as width of $f$.
 Finally, apply  Lemma~\ref{lemma6} to obtain 
  \begin{equation*}
  B_{\vect{x}\rightarrow \vect{y}}(\cNN) \leq  \Big(\big(t-1\big)\omega_f+1\Big)^{d_f}-1.
  \end{equation*}
\end{proof}

 \begin{proposition}{\label{proposition2}}
 Let $\mathcal R$ be a convex region in $\mathbb R^n$.  For any affine $\varepsilon$-approximation 
 $f :\mathcal{R} \rightarrow \mathbb{R}$ 
 of a function $g\in C^2({\mathcal{R}})$ we have
 %If $ S(g_{[\vect{x},\vect{y}]}) $ stands for the number of linear pieces that is used for $\varepsilon$ linear approximation of $f_{[\vect{x},\vect{y}]}$, then\\
   \begin{equation}
     B_{\vect{x}\rightarrow \vect{y}}(f) \geq \frac{||\vect{x} -\vect{y} ||_2}{4\sqrt{\varepsilon}} \cdot \Psi(g,\vect{x},\vect{y})  -1
 \end{equation}
 where $\Psi(g,\vect{x},\vect{y})$ is defined in \eqref{psii}.

 %\begin{align*}
 %\Psi(g,\vect{x},\vect{y})&\defeq \sqrt{\inf \limits_{0 \leq \alpha \leq 1}^{}\Big ( {\max \big\{0, \gamma(\alpha)\delta(\alpha) \big\} }\Big )} \\
 %\gamma(\alpha)&\defeq\min\big\{|\alpha_1(\alpha)|,|\alpha_2(\alpha)|\big\}\\
 %\delta(\alpha)&\defeq \mathrm{sign}\big(\alpha_1(\alpha)\alpha_2(\alpha)\big)
 %\end{align*}
 %and where $\alpha_1(\alpha)$ and $\alpha_2(\alpha)$ denote the largest and smallest eigenvalues of the hessian matrix %$\nabla^2g\big((1-\alpha)\vect{x}+\alpha\vect{y}\big)$, respectively.
 
 \end{proposition}
 %\subsection*{Proof of Proposition \ref{proposition2}}
  \begin{proof}[Proof of Proposition \ref{proposition2}] 
We partition $\mathcal{R}$ into \emph{convex} subregions $\mathcal{R}_i$, such that in each subregion  
$f(\vect x)$ is an affine function. 
%Consider segment $[\vect{x} ,\vect{y} ]$. As $(\vect{x} ,\vect{y} ) \in \mathcal{R}^2$ so $[\vect{x} ,\vect{y} ]$ lies completely in $\mathcal{R}$. 
These convex subregions partition a segment $[\vect{x} ,\vect{y} ]$ into  sub-segments with end points 
$\Big\{\vect{x} _0,\vect{x} _1,...,\vect{x} _s \Big \}$, where $\vect{x} _0=\vect{x} , \vect{x} _s=\vect{y} $ and  $s = B_{\vect{x} \rightarrow \vect{y}}(f)+1$.
  %see definition \ref{def4},\ref{def3} for definiton of  $S(g_{[\vect{x},\vect{y}]})$. 
In the sub-segment $i\in \{0,1,...,s-1\}$, 
 \begin{equation} \label{eq4}
f(\vect{x} )=\vect{p}_i .\vect{x} +q_i, \quad  \vect{x}  \in [\vect{x} _i,\vect{x} _{i+1}], 
 \end{equation}
 for some $\vect p_i$ and $\vect q_i$.  Let  $\vect{x} _i(r)=(1-r)\vect{x} _i+r\vect{x} _{i+1}$, 
 $r \in [0,1] $, and define
\begin{align*}
 &f_i(r)=(1-r)g(\vect{x} _i)+rg(\vect{x} _{i+1}),\\
 &h_i(r)=g\big (\vect{x} _i(r) \big),\\
& l_i(r) =f(\vect{x}(r)). 
     \end{align*}
      From the definition of $\varepsilon$-approximation, $||h_i(r)-l_i(r) ||_{\infty} \leq \varepsilon$. Thus
     \begin{align} 
  ||f_i(r)-&h_i(r) ||_{\infty} \leq ||f_i(r)- l_i(r)  ||_{\infty} + ||l_i(r) -h_i(r)  ||_{\infty} \notag\\
     &\overset{(a)}{\leq} \max\bigl\{|f_i(0)-l_i(0)|, |f_i(1)-l_i(1)|\bigr\} +\varepsilon \notag
     \\
     &\leq 2\varepsilon, \label{rel4}  
     \end{align}
     where $||k(r)||_{\infty}=\sup \limits_{0\leq r \leq 1}^{}k(r)$ and step $(a)$ follows because $f_i(r)$ and $l_i(r)$ are both line segments and 
     the maximum distance between them is achieved 
     at end points. 
     
  %From definition of $h_i(r)$ we have
 %\begin{align*}
  %   &h'_i(r)=\nabla g\big( \vect{x} _i(r)\big).(\vect{x} _{i+1}-\vect{x} _i)\\ 
   %  &h''_i(r)=(\vect{x} _{i+1}-\vect{x} _i)^T \nabla^2 g\big(\vect{x} _i(r)\big).(\vect{x} _{i+1}-\vect{x} _i)
% \end{align*}
 
 As $h(r)$ on $(0,1)$ is differentiable so there exists $r^{*}_i \in (0,1)$ such that $h'_i(r^{*}_i)=h_i(1)-h_i(0)$. Consider $\vect{x} ^{*}_i=(1-r^{*}_i)\vect{x} _i+r^{*}_i\vect{x} _{i+1}$. 
 From \eqref{rel4} we obtain
\begin{align*}
  %  & |(1-r_i^{*})g(\vect{x} _i)+r^{*}_i g(\vect{x} _{i+1})-g(\vect{x} ^{*}_{i})| =\\
     &  |(1-r_i^{*})\big ( g(\vect{x} _i)- g(\vect{x} _{i+1})\big) -g(\vect{x} ^{*}_{i})+g(\vect{x} _{i+1})| \leq 2 \varepsilon, \\
     &|r_i^{*}\big ( g(\vect{x} _{i+1})- g(\vect{x} _{i})\big) +g(\vect{x} _{i})-g(\vect{x} _{i}^{*})| \leq 2 \varepsilon.   
 \end{align*}

Then, from the definition of $r^{*}_i$ we have
 \begin{equation} \label{rel5}
     \begin{aligned}
      &  |(r_i^{*}-1)\nabla g(\vect{x} ^{*}_i).(\vect{x} _{i+1}-\vect{x} _{i}) -g(\vect{x} ^{*}_{i})+g(\vect{x} _{i+1})| \leq 2 \varepsilon 
    \end{aligned}
    \end{equation}
        \begin{equation} \label{rel6}
   \begin{aligned} 
      &  |r_i^{*}\nabla g(\vect{x} ^{*}_i).(\vect{x} _{i+1}-\vect{x} _{i}) -g(\vect{x} ^{*}_{i})+g(\vect{x} _{i})| \leq 2 \varepsilon .
     \end{aligned}
 \end{equation}
Since $g\in C^2({\mathcal{R}})$ a Taylor expansion of $g(\vect{x}_i)$ and $g(\vect{x}_{i+1})$ around $x^*_{i}$ gives
% \begin{align*}
%&h_i(0)=h_i(r^{*}_i)+h'_i(r^{*}_i)(0-r^{*}_i)+\frac{1}{2}h''_i(\alpha_i)(0-r^{*}_i)^2,\\
%&h_i(1)=h_i(r^{*}_i)+h'_i(r^{*}_i)(1-r^{*}_i)+\frac{1}{2}h''_i(\beta_{i})(1-r^{*}_i)^2
%\end{align*}
%
%Expressing $h_i'(r)$ and $h_i''(r)$ in terms of $g(x_i)$ we get
 \begin{align*}
&g(\vect{x} _i)=g(\vect{x} _i^{*})-r^{*}_i \nabla g\big( \vect{x} ^{*}_i\big).(\vect{x} _{i+1}-\vect{x} _i) 
\\ &+\frac{{r^{*}_i}^2}{2}(\vect{x} _{i+1}-\vect{x} _i)^T \nabla^2 g\big( \vect{x}_{i}(\alpha_i)\big)(\vect{x} _{i+1}-\vect{x} _i), \\
 &g(\vect{x} _{i+1})=g(\vect{x} _i^{*})+(1-r^{*}_i) \nabla g\big( \vect{x} ^{*}_i\big).(\vect{x} _{i+1}-\vect{x} _i) \\ &+\frac{{(1-r^{*}_i)}^2}{2}(\vect{x} _{i+1}-\vect{x} _i)^T \nabla^2 g\big(\vect{x}_i(\beta_i) \big)(\vect{x} _{i+1}-\vect{x} _i),
 \end{align*}
where $0 \leq \alpha_i \leq r^{*}_{i} \leq \beta_{i} \leq 1$.
 
 Substituting the above relations in inequalities \eqref{rel5} and ~\eqref{rel6} we get
 \begin{equation} \label{eq5}
     |{(1-r^{*}_i)}^2(\vect{x} _{i+1}-\vect{x} _i)^T \nabla^2 g\big(\vect{x}_{i}(\beta_i) \big)(\vect{x} _{i+1}-\vect{x} _i)| \leq 4\varepsilon, 
 \end{equation}
 \begin{equation} \label{eq6}
     |{r^{*}_i}^2(\vect{x} _{i+1}-\vect{x} _i)^T \nabla^2 g\big(\vect{x}_{i}(\alpha_i) \big)(\vect{x} _{i+1}-\vect{x} _i)| \leq 4\varepsilon .
 \end{equation}
 Use the \emph{Rayleigh quotient} and the definitions of $\theta(\alpha),\gamma(\alpha)$ to obtain
  \begin{align*}
 &|\frac{(\vect{x} _{i+1}-\vect{x} _i)^T \nabla^2 g\big(\vect{x}_{i}(\alpha_i) \big)(\vect{x} _{i+1}-\vect{x} _i)}{(\vect{x} _{i+1}-\vect{x} _{i})^T (\vect{x} _{i+1}-\vect{x} _{i})}| \\
 &\geq \inf \limits_{0 \leq \alpha \leq 1}^{}\Big ( {\max \big\{0, \theta(\alpha)\gamma(\alpha) \big\} } \Big).
 \end{align*}
 Combining the above inequality with ~\eqref{eq5} and ~\eqref{eq6} and the fact that ${r^{*}_{i}}^2+(1-r^{*}_{i})^2 \geq \frac{1}{2}$ we get 
 \begin{align*}
 {{||\vect{x} _{i+1}-\vect{x} _{i}||_2}^2}.\inf \limits_{0 \leq \alpha \leq 1}^{}\Big ( {\max \big\{0, \theta(\alpha)\gamma(\alpha) \big\} } \Big) \leq 16 \varepsilon  . 
 \end{align*}
 Accordingly,
 \begin{align*}
 \sum \limits_{i=0}^{s-1} \Bigg ( \frac{{||\vect{x} _{i+1}-\vect{x} _{i}||_2}}{4\sqrt{\varepsilon}}.\sqrt{\inf \limits_{0 \leq \alpha \leq 1}^{}\Big ( {\max \big\{0, \theta(\alpha)\gamma(\alpha) \big\} } \Big)} \Bigg) \leq s,
%&\Rightarrow \Bigg( \frac{\sqrt{\inf \limits_{0 \leq \alpha \leq 1}^{}\Big ( {\max \big\{0, \theta(\alpha)\gamma(\alpha) \big\} } \Big)}}{4\sqrt{\varepsilon}}\sum \limits_{i=0}^{s-1}  {{||\vect{x} _{i+1}-\vect{x} _{i}||_2}} \Bigg) \\ 
 %&\leq B_{\vect{x} \rightarrow \vect{y}}(f) + 1 \\  
 \end{align*}
which gives
 \begin{align*}
B_{\vect{x} \rightarrow \vect{y}}(f) \geq   \frac{{||\vect{x} -\vect{y} ||_2}}{4\sqrt{\varepsilon}}\Psi(g,\vect{x},\vect{y}) -1 .  
 \end{align*}
 \end{proof}
 \begin{proposition}{\label{proposition3}} 
 Let $g:[0,1]^n \rightarrow \mathbb{R}$ be such that $D^{J}(g)(\vect{x}) \leq \delta$ for any $\vect{x} \in [0,1]^n  $ and any multi-index $J $ such that $|J|=3$. Then, for any affine $\varepsilon$-approximation $f$ 
 %
 %\begin{equation} \label{eq3}
 %B_{\vect{x} \rightarrow \vect{y}} (f) \geq  \cdot \sqrt{\frac{\Big ( \max \limits_{\vect{x} \in [0,1]^n}^{} \rho{\big(\nabla^2 g(\vect{x})\big)}- \delta \cdot n^\frac{3}{2} \Big)^+}{16\varepsilon}}-1  
 %\end{equation}

 %for any $\vect{x},\vect{y} \in [0,1]^n$ where $\rho(\cdot)$ denotes the spectral radius .\footnote{$a^+\defeq \max{\{0,a\}}$ }
 
 %  A weaker form of \eqref{eq3} is 
 %
 %
 \begin{align*}
 B_{\vect{x} \rightarrow \vect{y}} (f) \geq   \sqrt{ \frac{\Big ( \max \limits_{ \vect{x} \in [0,1] ^ n } ^ {}\big| {\Delta(g)(\vect{x})}\big| \cdot n^{-1}  - \delta \cdot n^\frac{3}{2} \Big)^+}{16\varepsilon}}-1  
 \end{align*} 
 for any $\vect{x},\vect{y} \in [0,1]^n$, where $\Delta$ denotes the Laplace operator \eqref{eq:laplacian}.
  \end{proposition}
% \subsection*{Proof of Proposition~\ref{proposition3}}
 \begin{proof}[Proof of Proposition~\ref{proposition3}]
 Define
 \begin{equation*}
\vect{z} \defeq \argmax \limits_{ \vect{x} \in \mathcal{R} }^{}{ \rho \big( \nabla^2 g(\vect{x})  \big) }
 \end{equation*}
 where $\rho(\cdot)$ denotes the spectral radius. Let $\vect{u}$ be a normalized eigenvector corresponding to
 an eigenvalue $\lambda$ where $|\lambda|=\rho \big( \nabla^2 g(\vect{z})  \big)$, {\it{i.e.}}, 
 \begin{equation}
 \nabla^2 g(\vect{z}) \vect{u}  = \lambda \vect{u},\quad ||\vect u||=1.
 \end{equation}
% Where and from definition of $\vect{z}$, $\rho \big( \nabla^2 g(\vect{z})\big) = \max \limits_{\vect{x} \in [0,1]^n}^{} \rho{\big(\nabla^2 g(\vect{x})\big)}$. 
Consider any segment $[\vect{x},\vect{y}]$ in $\mathcal{R}$ in the direction of $\vect u$, {\it{i.e.}}, such that ${\vect{x}-\vect{y}}=\vect{u}$.
%\footnote{For every given direction $\vect{u}$ we can find a segment in $[0,1]^n$ with  length $1$ in the direction of $\vect{u}$  }. Consider the segment $[\vect{x},\vect{y}]$. 
The convex subregions of $f$, defined in the proof of Proposition \ref{proposition2}, divide this segment into sub-segments with end points $\{\vect{x}_0, \vect{x}_1,...,\vect{x}_s \}$ where $\vect{x}_0=\vect{x}, \vect{x}_s=\vect{y}$ and $s=B_{\vect{x}\rightarrow \vect{y}} (f) + 1$. 
%On  
%sub-segment $i \in \{0,1,...,s-1\}$, $f(x)$ is an affine function:  
 %\begin{align*}
  %f(\vect{x} )=\vect{p}_i .\vect{x} +q_i,\quad     \vect{x}  \in [\vect{x} _i,\vect{x} _{i+1}].
 %\end{align*}   
  Using the same analysis as in the proof of Proposition \ref{proposition2}, from \eqref{eq4}--\eqref{eq6} we obtain \eqref{eq5} and \eqref{eq6}.
 %\begin{equation}  \label{eq7}
   %  |{(1-r^{*}_i)}^2(\vect{x} _{i+1}-\vect{x} _i)^T \nabla^2 g\big(\vect{x}_{i}(\beta_i) \big)(\vect{x} _{i+1}-\vect{x} _i)| \leq 4\varepsilon 
 %\end{equation}
 %\begin{equation}  \label{eq8}
  %   |{r^{*}_i}^2(\vect{x} _{i+1}-\vect{x} _i)^T \nabla^2 g\big(\vect{x}_{i}(\alpha_i) \big)(\vect{x} _{i+1}-\vect{x} _i)| \leq 4\varepsilon 
 %\end{equation}
 On the other hand, note that
 \begin{align*}
 &|(\vect{x} _{i+1}-\vect{x} _i)^T \nabla^2 g\big(\vect{x}_{i}(\alpha_i) \big)(\vect{x} _{i+1}-\vect{x} _i)|\\
 & \geq |(\vect{x} _{i+1}-\vect{x} _i)^T \nabla^2 g\big(\vect{z} \big)(\vect{x} _{i+1}-\vect{x} _i)| \\
 &-  |(\vect{x} _{i+1}-\vect{x} _i)^T \Big( \nabla^2 g\big(\vect{x}_{i}(\alpha_i) \big) -  \nabla^2 g\big(\vect{z} \big) \Big)   (\vect{x} _{i+1}-\vect{x} _i)| \\
 &=|\lambda|\cdot || \vect{x}_{i+1}-\vect{x}_i ||^2 \\
 &-\big|\mathrm{tr} \big\{ \big( \nabla^2 g\big(\vect{x}_{i}(\alpha_i) \big) - \nabla^2 g\big(\vect{z} \big) \big) (\vect{x}_{i+1}-\vect{x}_i)(\vect{x}_{i+1}-\vect{x}_i)^T           \big\} \big| \\
  &\overset{(a)}{\geq} |\lambda|\cdot || \vect{x}_{i+1}-\vect{x}_i ||^2 \\ 
 &-\big|\big|  \nabla^2 g\big(\vect{x}_{i}(\alpha_i) \big) - \nabla^2 g\big(\vect{z} \big)      \big|\big|_{\mathrm{F}}  
 \big|\big|  (\vect{x}_{i+1}-\vect{x}_i)(\vect{x}_{i+1}-\vect{x}_i)^T    \big|\big|_{\mathrm{F}} \\ 
%  &|(\vect{x} _{i+1}-\vect{x} _i)^T \nabla^2 g\big(\vect{x}_{i}(\alpha_i) \big)(\vect{x} _{i+1}-\vect{x} _i)|\\
% & \geq  |\lambda| \cdot || \vect{x}_{i+1}-\vect{x}_i ||^2  -\\
% & ||  \nabla^2 g\big(\vect{x}_{i}(\alpha_i) \big) - \nabla^2 g\big(\vect{z} \big)      ||_{\mathrm{F}}  
% ||  (\vect{x}_{i+1}-\vect{x}_i)(\vect{x}_{i+1}-\vect{x}_i)^T    ||_{\mathrm{F}} \\ 
 &= |\lambda| \cdot || \vect{x}_{i+1}-\vect{x}_i ||^2\\
 &- \big|\big|  \nabla^2 g\big(\vect{x}_{i}(\alpha_i) \big) - \nabla^2 g\big(\vect{z} \big)      \big|\big|_{\mathrm{F}}  
 ||\vect{x}_{i+1} -\vect{x}_i ||^2\\
 %&FORMORESPACE \geq |\lambda| \cdot || \vect{x}_{i+1}-\vect{x}_i ||^2 - n \delta \cdot || \vect{z}-\vect{x}_{i}(\alpha_i)|| \cdot||\vect{x}_{i+1} -\vect{x}_i ||^2 \\
 &=||\vect{x}_{i+1} -\vect{x}_i ||^2 \cdot \Big(|\lambda| -   n \delta \cdot || \vect{z}-\vect{x}_{i}(\alpha_i)||    \Big) \\
 &\geq  ||\vect{x}_{i+1} -\vect{x}_i ||^2 \cdot \Big(|\lambda| -  \delta \cdot n^{\frac{3}{2}}    \Big),
 \end{align*}
where in step $(a)$ we used the inequality 
 \begin{align*}
 \Big|\mathrm{tr}\big( AB  \big) \Big| \leq ||A||_F ||B||_F,
 \end{align*} 
  $||\cdot||_F$ stands for Frobenius norm.
 
Combining the above relation with ~\eqref{eq5},~\eqref{eq6} and the fact that ${r^{*}_i}^2+(1-r^{*}_i)^2 \geq \frac{1}{2}$ we get
  \begin{align*}
 16\varepsilon \geq ||\vect{x}_{i+1} -\vect{x}_i ||^2 \cdot \Big(|\lambda|-  \delta \cdot n^{\frac{3}{2}}    \Big), 
  \end{align*} 
%  which gives
 % \begin{align*}
 %4 \sqrt{\varepsilon} \geq ||\vect{x}_{i+1} -\vect{x}_i || \cdot \sqrt{ \Big(|\lambda|-  \delta \cdot n^{\frac{3}{2}}    \Big)^{+} },
 %\end{align*} 
 which gives
  \begin{align*}
  4 \sqrt{\varepsilon} \cdot\big( B_{\vect{x} \rightarrow \vect{y}}(f)+1 \big)\geq ||\vect{x}-\vect{y}|| \cdot \sqrt{ \Big(|\lambda|-  \delta \cdot n^{\frac{3}{2}}    \Big)^{+} } .
  \end{align*} 
  Finally, rewriting the above inequality we get
   \begin{align*}
B_{\vect{x} \rightarrow \vect{y}} (f) \geq \frac{1}{4\sqrt{\varepsilon}} \cdot \sqrt{\Big ( |\lambda|- \delta \cdot n^\frac{3}{2} \Big)^+}-1.
 \end{align*} 
 
 Since $|\lambda|=\rho \big( \nabla^2 g(\vect{z})  \big)=\max \limits_{\vect{x} \in [0,1]^n}^{} \rho{\big(\nabla^2 g(\vect{x})\big)}$ and  $$|\Delta(g)(\vect{x})|=|\mathrm{tr}(\nabla^2 g(\vect{x}))| \leq \rho(\nabla^2 g(\vect{x})) \cdot n,  $$
 we obtain the desired result.
 \end{proof}

\subsection*{Proofs of Theorems~\ref{theorem1} and~\ref{theorem2}}
Propositions~\ref{proposition1} and ~\ref{proposition2} give Theorem~\ref{theorem1} and Propositions~\ref{proposition1} and ~\ref{proposition3} give Theorem~\ref{theorem2}.\qed
 \subsection*{Proof of Theorem \ref{theorem3}}\label{pfth4}
 Given a neural network $f$ we use $o$ to denote the output unit, $\mathrm{w}(u,v)$ to denote the weight of two connected units $u$ and $v$, and $b(u)$ to denote the bias of unit $u$. Furthermore, given $u \in \cH$ and $\vect{x} \in \mathcal{R}$ let $f_1^u(\vect{x})$ denote the output of unit $u$ when the input to $f_1$ is $\vect{x}$, and similarly for $f_2(\vect{x})$. Finally, define the maximum change in hidden layer $i$ as
 \begin{align*}
 \varepsilon_{i}(\vect{x})\defeq \max \limits_{u \in \cH^{i} }^{} \Big\{ |f^{u}_1(\vect{x})-f^{u}_2(\vect{x})| \Big \}.
 \end{align*}
 Fix $1\leq i\leq d_f-1$ and $v \in \cH^{i+1}$. Then, 
 \begin{align*}
 &\big|f^{v}_1(\vect{x})-f^{v}_2(\vect{x})\big| \\
 &=\Bigg|\sigma_1 \Big( \sum\limits_{u \in \bigcup\limits_{j=1}^{i} \cH^j } ^{}w(u,v)\cdot f^{u}_1(\vect{x})+b(v)\Big) \\
 &- \sigma_2 \Big( \sum\limits_{u \in \bigcup\limits_{j=1}^{i} \cH^j } ^{}w(u,v)\cdot f^{u}_2(\vect{x})+b(v)\Big)\Bigg|\\
 & \leq \varepsilon + \delta \cdot \Big ( \sum\limits_{u \in \bigcup \limits_{j=1}^{i}\cH^j }^{} |w(u,v)|\cdot\big|f^{u}_1(\vect{x})-f^{u}_{2}(\vect{x})\big|  \Big) \\
 & \leq  \varepsilon + \delta A\cdot \Big ( \sum\limits_{j=1}^{i} \sum\limits_{u \in \cH^{j}}^{} \big|f^{u}_{1}(\vect{x})-f^{u}_2(\vect{x})\big|  \Big) \\
 &\leq 
 %\varepsilon + \delta A\cdot \Big ( \sum\limits_{j=1}^{i} \sum\limits_{u \in \cH^{j}}^{} \varepsilon_{j}(\vect{x})  \Big) =
 \varepsilon + \delta A\cdot \Big ( \sum\limits_{j=1}^{i} \omega_j \varepsilon_{j}(\vect{x})  \Big) 
 \end{align*}
 where the first inequality holds since $\sigma_1$ is $\delta$-Lipschitz and assuming that $ ||\sigma_1 - \sigma_2||_{\infty}\leq \varepsilon$. 
 Hence we get the recursion between $\varepsilon_i$'s
 \begin{align}\label{recursion}
 &\varepsilon_{i+1}(\vect{x}) \leq  \varepsilon + \delta A \cdot \Big ( \sum\limits_{j=1}^{i} \omega_j \varepsilon_{j}(\vect{x})  \Big) 
 \end{align}
 for $1\leq i\leq d_f-1.$
 Now, since $\varepsilon_1(\vect{x}) \leq \big|\sigma_1(\vect{x})-\sigma_2(\vect{x})\big|$ we get $\varepsilon_1(\vect{x}) \leq \varepsilon$. From this initial condition and \eqref{recursion} 
  \iffalse
  \begin{equation}
  \begin{aligned}
 &\varepsilon_{i+1}(\vect{x}) \leq  \varepsilon + \delta A \Big ( \sum\limits_{j=1}^{i} \omega_j \varepsilon_{j}(\vect{x})  \Big) \\
 & \varepsilon_1(\vect{x}) \leq \varepsilon
  \end{aligned}
  \end{equation}
 \fi
 \begin{equation} \label{rel7}
 \varepsilon_{i+1}(\vect{x}) \leq  \varepsilon (1+\delta A\omega_1)(1+\delta A \omega_2)\cdots(1+\delta A \omega_i). 
 % &\Rightarrow \varepsilon_{d}(\vect{x}) \leq \varepsilon (1+\delta Al_1)(1+\delta A l_2)...(1+\delta A l_{d-1})\\
 \end{equation}
 On the other hand we have
  \begin{align*}
 & |f_{1}(\vect{x}) - f_{2}(\vect{x}) |=\Big| \sum\limits_{u \in \bigcup\limits_{j=1}^{d_f} \cH^j } ^{}\mathrm{w}(u,o)\cdot \big( f^{u}_1(\vect{x})-
   f^{u}_{2}(\vect{x}) \big) \Big|\\
 & \leq A\big( \varepsilon_1(\vect{x})\omega_1 +\varepsilon_2(\vect{x})\omega_2+...+\varepsilon_d(\vect{x})\omega_{d_f} \big)
 \end{align*}
and from ~\eqref{rel7} we finally get
 \begin{align*}
 &|f_1(\vect{x}) - f_2(\vect{x}) | \\
&  \leq  \frac{\varepsilon}{\delta} \Big((1+\delta A\omega_1)(1+\delta A\omega_2)...(1+\delta A\omega_{d_f})-1\Big) \\
&\leq \frac{||\sigma_1 -\sigma_2||_{\infty}}{\delta} \Bigg( \Big(\delta\cdot A\cdot \omega_f  +1\Big)^{d_f} -1 \Bigg )
\end{align*}
 which gives the desired result.\qed

\bibliography{example_paper}
\bibliographystyle{icml2018}

\end{document}